%% file: paper.tex
\documentclass{article}

\usepackage{graphicx} 
\usepackage{subfigure} 

\usepackage{natbib}

\usepackage{amsmath}
\input{def.tex}

 \usepackage[accepted]{icml2012}

\icmltitlerunning{Coactive Learning}

\begin{document} 

\twocolumn[
\icmltitle{Online Structured Prediction via Coactive Learning}

\icmlauthor{Pannaga Shivaswamy}{pannaga@cs.cornell.edu}
\icmlauthor{Thorsten Joachims}{tj@cs.cornell.edu}
\icmladdress{Department of Computer Science, Cornell University, Ithaca NY 14853}

\icmlkeywords{Online Learning, Structured Prediction, Web-search, Recommender Systems}

\vskip 0.3in
]

\begin{abstract}
We propose Coactive Learning as a model of interaction between a learning system and a human user, where both have the common goal of providing results of maximum utility to the user. At each step, the system (e.g. search engine)  receives a context (e.g. query) and predicts an object (e.g. ranking).  The user responds by correcting the system if necessary, providing a slightly improved -- but not necessarily optimal -- object as feedback.  We argue that such feedback can often be inferred from observable user behavior, for example, from clicks in web-search. Evaluating predictions by their cardinal utility to the user, we propose efficient learning algorithms that have ${\cal O}(\frac{1}{\sqrt{T}})$ average regret, even though the learning algorithm never observes cardinal utility values as in conventional online learning. We demonstrate the applicability of our model and learning algorithms on a movie recommendation task, as well as ranking for web-search.

\end{abstract}

\input{intro.tex}
\input{related.tex}

\input{model.tex}
\input{linutil.tex}

\input{extensions.tex}

\input{expts.tex}

\section{Conclusions}
We proposed a new model of online learning where preference feedback is observed but cardinal feedback is never observed. We proposed a suitable notion of regret and showed that it can be minimized under our feedback model. Further, we provided several extensions of the model and algorithms. Furthermore, experiments demonstrated its effectiveness for web-search ranking and a movie recommendation task.  A future direction is to consider $\lambda$-strongly convex functions, and we conjecture it is possible to derive algorithms with ${\cal O}(\log(T)/T)$ regret in this case.

\paragraph{Acknowledgements} We thank Peter Frazier, Bobby Kleinberg, Karthik Raman and Yisong Yue for helpful discussions. This work was funded in part under NSF awards IIS-0905467 and IIS-1142251. 
\begin{small}
\bibliography{all}
\bibliographystyle{icml2012}
\end{small}
\end{document}

%% file: def.tex
\DeclareMathOperator*{\argmax}{arg\,max}
\def\bx{{\bf x}}
\def\by{{\bf y}}
\def\bby{{\bf \bar{y}}}
\def\bbw{{\bf \bar{w}}}
\def\bw{{\bf w}}
\def\bu{{\bf u}}
\def\bm{{\bf m}}

\def\be{{\bf e}}
\def\argmax{\text{argmax}}

\def \nn {\nonumber} 
\def \nnnl {\nonumber \\}

\newcommand{\beqn}{\begin{equation}}
\newcommand{\eeqn}{\end{equation}}
\newcommand{\bitm}{\begin{itemize}}
\newcommand{\eitm}{\end{itemize}}
\newcommand{\bfig}{\begin{figure}}
\newcommand{\efig}{\end{figure}}
\newcommand{\baln}{\begin{align}}
\newcommand{\ealn}{\end{align}}
\newcommand{\btblr}{\begin{tabular}}
\newcommand{\etblr}{\end{tabular}}
\newcommand{\btbl}{\begin{table}}
\newcommand{\etbl}{\end{table}}
\newcommand{\bsec}{\begin{section}}
\newcommand{\esec}{\end{section}}
\newcommand{\bchr}{\begin{chapter}}
\newcommand{\echr}{\end{chapter}}
\newcommand{\bseqn}{\begin{subequations}}
\newcommand{\eesqn}{\end{subequations}}

\newcommand{\inner}[2]{ #1^\top #2}
\newcommand{\expct}{\mathbf E}
\newcommand{\pr}{\mathbf P}

\newcommand{\RR}{\mathbf{R}}
\newcommand{\BlackBox}{\rule{1.5ex}{1.5ex}}  
\newenvironment{proof}{\par\noindent{\bf Proof\ }}{\hfill\BlackBox\\[2mm]}

\newtheorem{theorem}{Theorem}
\newtheorem{lemma}[theorem]{Lemma}

\newtheorem{corollary}[theorem]{Corollary}

%% file: intro.tex
\section{Introduction}
In a wide range of systems in use today, the interaction between human and system takes the following form. The user issues a command (e.g. query) and receives a -- possibly structured -- result in response (e.g. ranking). The user then interacts with the results (e.g. clicks), thereby providing implicit feedback about the user's utility function. Here are three examples of such systems and their typical interaction patterns:
\vspace{-0.7\baselineskip}
\begin{description}\addtolength{\itemsep}{-0.55\baselineskip}
\item[Web-search:] In response to a query, a search engine presents the ranking $[A,B,C,D,...]$ and observes that the user clicks on documents $B$ and $D$.
\item[Movie Recommendation:] An online service recommends movie A to a user. However, the user rents movie B after browsing the collection.
\item[Machine Translation:] An online machine translator is used to translate a wiki page from language A to B. The system observes some corrections the user makes to the translated text.
\end{description}
\vspace{-0.7\baselineskip}
In all the above examples, the user provides some feedback about the results of the system. However, the feedback is only an incremental improvement, not necessarily the optimal result. For example, from the clicks on the web-search results we can infer that the user would have preferred the ranking $[B,D,A,C,...]$ over the one we presented. However, this is unlikely to be the best possible ranking. Similarly in the recommendation example, movie $B$ was preferred over movie $A$, but there may have been even better movies that the user did not find while browsing. In summary, the algorithm typically receives a slightly improved result from the user as feedback, but not necessarily the optimal prediction nor any cardinal utilities. We conjecture that many other applications fall into this schema, ranging from news filtering to personal robotics.

Our key contributions in this paper are threefold. First, we formalize Coactive Learning as a model of interaction between a learning system and its user, define a suitable notion of regret, and validate the key modeling assumption -- namely whether observable user behavior can provide valid feedback in our model -- in a web-search user study. Second, we derive learning algorithms for the Coactive Learning Model, including the cases of linear utility models and convex cost functions, and show ${\cal O}(1/\sqrt{T})$  regret bounds in either case with a matching lower bound. The learning algorithms perform structured output prediction (see \cite{bakir}) and thus can be applied in a wide variety of problems.  Several extensions of the model and the algorithm are discussed as well. Third, we provide extensive empirical evaluations of our algorithms on a movie recommendation and a web-search task, showing that the algorithms are highly efficient and effective in practical settings.

%% file: related.tex
\section{Related Work}

The Coactive Learning Model bridges the gap between two forms of feedback that have been well studied in online learning. On one side there is the multi-armed bandit model \cite{ACSF02,AuerCF02}, where an algorithm chooses an action and observes the utility of (only) that action. On the other side, utilities of all possible actions are revealed in the case of learning with expert advice \cite{olbook}. Online convex optimization \cite{Zink03} and online convex optimization in the bandit setting \cite{FKM05} are continuous relaxations of the expert and the bandit problems respectively. Our model, where information about two arms is revealed at each iteration sits between the expert and the bandit setting. Most closely related to Coactive Learning is the dueling bandits setting \cite{Yue/etal/09a,Yue/Joachims/09a}. The key difference is that both arms are chosen by the algorithm in the dueling bandits setting, whereas one of the arms is chosen by the user in the Coactive Learning setting.

While feedback in Coactive Learning takes the form of a preference, it is different from ordinal regression and ranking. Ordinal regression \cite{prank} assumes training examples $(x,y)$, where $y$ is a rank. In the Coactive Learning model, absolute ranks are never revealed. Closely related is learning with pairs of examples \cite{Her2000,FreIye03,ChuGha05} where absolute ranks are not needed; however, existing approaches require an {\em iid} assumption and typically perform batch learning.  There is also a large body of work on ranking (see \cite{Liu:2009}). These approaches are different from Coactive Learning; they require training data $(x,y)$ where $y$ is the {\em optimal} ranking for query $x$. 

%% file: model.tex
\section{Coactive Learning Model}
We now introduce coactive learning as a model of interaction (in rounds) between a learning system (e.g. search engine) and a human (e.g. search user) where both the human and learning algorithm have the same goal (of obtaining good results). At each round $t$, the learning algorithm observes a context $\bx_t \in {\cal X}$ (e.g. a search query) and presents a structured object $\by_t \in {\cal Y}$  (e.g. a ranked list of URLs). The utility of $\by_t \in {\cal Y}$ to the user for context $\bx_t \in {\cal X}$ is described by a utility function $U(\bx_t,\by_t)$, which is unknown to the learning algorithm. As feedback the  user returns an improved object $\bby_t \in {\cal Y}$ (e.g. reordered list of URLs), i.e., 
\begin{align}
\label{eq:pref-feedback}
U(\bx_t,\bby_t) > U(\bx_t,\by_t),
\end{align}
when such an object $\bby_t$ exists. In fact, we will also allow violations of  (\ref{eq:pref-feedback})  when we formally model user feedback in Section \ref{ss:feedback}.
The process by which the user generates the feedback $\bby_t$ can be understood as an approximate utility-maximizing search, but over a user-defined subset ${\cal \bar{Y}}_t$ of all possible ${\cal Y}$. This models an approximately and boundedly rational user that may employ various tools (e.g., query reformulations, browsing) to perform this search. Importantly, however, the feedback $\bby_t$ is typically not the optimal label 
\begin{align}
\label{eq:ystar}
\by_t^* := \argmax_{\by \in {\cal Y}} U(\bx_t, \by).
\end{align}
In this way, Coactive Learning covers settings where the user cannot manually optimize the $\argmax$ over the full ${\cal Y}$ (e.g. produce the best possible ranking in web-search), or has difficulty expressing a bandit-style cardinal rating for $\by_t$ in a consistent manner. This puts our preference feedback $\bby_t$ in stark contrast to supervised learning approaches which require $(\bx_t,\by_t^*)$. But even more importantly, our model implies that reliable preference feedback \eqref{eq:pref-feedback} can be derived from observable user behavior (i.e., clicks), as we will demonstrate in Section~\ref{sec:userstudy} for web-search. We conjecture that similar feedback strategies also exist for other applications, where users can be assumed to act approximately and boundedly rational according to $U$. 

Despite the weak preference feedback, the aim of a coactive learning algorithm is  to still present objects with utility close to that of the optimal $\by_t^*$. Whenever, the algorithm presents an object $\by_t$ under context $\bx_t$, we say that it suffers a regret $U(\bx_t,\by_t^*) - U(\bx_t,\by_t)$ at time step $t$. Formally, we consider the average regret suffered by an algorithm over $T$ steps as follows: 
\begin{align}
\label{eq:linregret}
REG_T = \frac{1}{T} \sum_{t=1}^T \left( U(\bx_t,\by_t^*) - U(\bx_t,\by_t) \right).
\end{align}
The goal of the learning algorithm is to minimize $REG_T$, thereby providing the human with predictions $\by_t$ of high utility. Note, however, that a cardinal value of $U$ is never observed by the learning algorithm, but $U$ is only revealed ordinally through preferences \eqref{eq:pref-feedback}.

\subsection{Quantifying Preference Feedback Quality}
\label{ss:feedback}
To provide any theoretical guarantees about the regret of a learning algorithm in the coactive setting, we need to quantify the quality of the user feedback. Note that this quantification is a tool for theoretical analysis, not a prerequisite or parameter to the algorithm. We quantify feedback quality by how much improvement $\bby$ provides in utility space. In the simplest case, we say that user feedback  is {\em strictly $\alpha$-informative} when the following inequality is satisfied:
\begin{align}
\label{eq:inf-feedback}
\!\!\!\!  U(\bx_t,\bby_t) - U(\bx_t,\by_t) \ge \alpha (U(\bx_t,\by_t^*) - U(\bx_t,\by_t)).
\end{align}
In the above inequality, $\alpha \in (0,1]$ is an unknown parameter. Feedback is such that utility of $\bby_t$ is higher than that of $\by_t$ by a fraction $\alpha$ of the maximum possible utility range $U(\bx_t,\by_t^*) - U(\bx_t,\by_t)$. Violations of the above feedback model are allowed by introducing slack variables $\xi_t \ge 0$:\footnote{Strictly speaking, the value of the slack variable depends on the choice of $\alpha$ and the definition of utility. However, for brevity, we do not explicitly show this dependence.}
\begin{align}
\label{eq:inf-feedback-relax}
\!\!\! U(\bx_t,\bby_t)\! - \! U(\bx_t,\by_t) \!  \ge \alpha (U(\bx_t,\!\by_t^*)\! -\! U(\bx_t,\!\by_t)\!)\!  - \! \xi_t.
\end{align}
We refer to the above feedback model as {\em $\alpha$-informative} feedback. Note also that it is possible to express feedback of any quality using \eqref{eq:inf-feedback-relax} with an appropriate value of $\xi_t$. Our regret bounds will contain $\xi_t$, quantifying to what extent the strict $\alpha$-informative modeling assumption is violated.

Finally, we will also consider an even weaker feedback model where a positive utility gain is only achieved in expectation over user actions: 
\begin{align}
\label{eq:exp-feedback-relax}
\!\!\!\!\!\!\!\! \expct_t [U\!(\!\bx_t,\!\bby_t\!)  \!-\! U\!(\!\bx_t,\!\by_t\!)] \!
\ge \! \alpha ( U\!(\!\bx_t,\!\by_t^*\!) \!-\! U\!(\!\bx_t,\!\by_t\!)\! ) \!-\!  \bar{\xi}_t.\!\!\!\!\!\!
\end{align}
We refer to the above feedback as {\em expected $\alpha$-informative} feedback. In the above equation, the expectation is over the user's choice of $\bby_t$  given $\by_t$ under context $\bx_t$ (i.e., under a distribution $\pr_{\bx_t}[\bby_t|\by_t]$ which is dependent on $\bx_t$). 

\subsection{User Study: Preferences from Clicks} \label{sec:userstudy}

\begin{figure}[t]
\begin{center}
\vskip -0.05in
\includegraphics[width=3.3in]{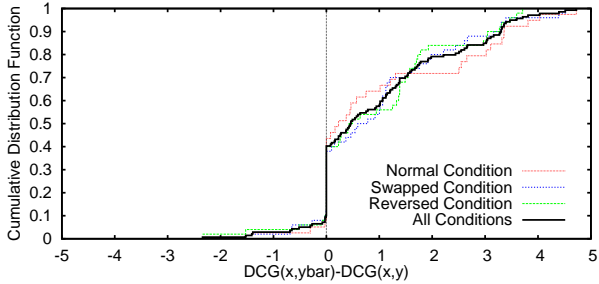}
\vskip -0.1in
\caption{\label{fig:cdf_alphainform} Cumulative distribution of utility differences between presented ranking $\by$ and click-feedback ranking $\bby$ in terms of $DCG@10$ for three experimental conditions and overall.}
\end{center}
\vskip -0.25in
\end{figure}

We now validate that reliable preferences as specified in Equation~\eqref{eq:pref-feedback} can indeed be inferred from implicit user behavior. In particular, we focus on preference feedback from clicks in web-search and draw upon data from a user study \cite{Joachims/etal/07a}. In this study, subjects (undergraduate students, $n=16$) were asked to answer 10 questions -- 5 informational, 5 navigational -- using the Google search engine. All queries, result lists, and clicks were recorded. For each subject, queries were grouped into query chains by question\footnote{This was done manually, but can be automated with high accuracy \cite{Jones/Klinkner/08}.}. On average, each query chain contained $2.2$ queries and $1.8$ clicks in the result lists.

We use the following strategy to infer a ranking $\bby$ from the user's clicks: prepend to the ranking $\by$ from the first query of the chain all results that the user clicked throughout the whole query chain. To assess whether $U(\bx,\bby)$ is indeed larger than $U(\bx,\by)$ as assumed in our learning model, we measure utility in terms of a standard measure of retrieval quality from Information Retrieval. We use $DCG@10(\bx,\by)=\sum_{i=1}^{10} \frac{r(\bx,\by[i])}{\log{i+1}}$, where $r(\bx,\by[i])$ is the relevance score of the i-th document in ranking $\by$ (see e.g. \cite{Manning/etal/08}). To get ground-truth relevance assessments $r(\bx,d)$, five human assessors were asked to manually rank the set of results encountered during each query chain. We then linearly normalize the resulting ranks to a relative relevance score $r(\bx,d) \in [0..5]$ for each document. 

We can now evaluate whether the feedback ranking $\bby$ is indeed better than the ranking $\by$ that was originally presented, i.e. $DCG@10(\bx,\bby) > DCG@10(\bx,\by)$. Figure~\ref{fig:cdf_alphainform} plots the Cumulative Distribution functions (CDFs) of $DCG@10(\bx,\bby)-DCG@10(\bx,\by)$ for three experimental conditions, as well as the average over all conditions. All CDFs are shifted far to the right of $0$, showing that preference feedback from our strategy is highly accurate and informative. Focusing first on the average over all conditions, the utility difference is strictly positive on $\sim 60\%$ of all queries, and strictly negative on only $\sim 10\%$. This imbalance is significant (binomial sign test, $p<0.0001$). Among the remaining $\sim 30\%$ of cases where the DCG@10 difference is zero, $88\%$ are due to $\bby = \by$ (i.e. click only on top 1 or no click). Note that a learning algorithm can easily detect those cases and may explicitly eliminate them as feedback. Overall, this shows that implicit feedback can indeed produce accurate preferences.

What remains to be shown is whether the reliability of the feedback is affected by the quality of the current prediction, i.e., $U(\bx_t,\by_t)$. In the user study, some users actually received results for which retrieval quality was degraded on purpose. In particular, about one third of the subjects received Google's top 10 results in reverse order (condition ``reversed'') and another third received rankings with the top two positions swapped (condition ``swapped''). As Figure~\ref{fig:cdf_alphainform} shows, we find that users provide accurate preferences across this substantial range of retrieval quality. Intuitively, a worse retrieval system may make it harder to find good results, but it also makes an easier baseline to improve upon. This intuition is formally captured in our definition of $\alpha$-informative feedback. The optimal value of the $\alpha$ vs.~$\xi$ trade-off, however, will likely depend on many application-specific factors, like user motivation, corpus properties, and query difficulty. In the following, we therefore present algorithms that do not require knowledge of $\alpha$, theoretical bounds that hold for any value of $\alpha$, and experiments that explore a large range of $\alpha$.



%% file: linutil.tex
\section{Coactive Learning Algorithms}
In this section, we present algorithms for minimizing regret in the coactive learning model. In the rest of this paper, we use a linear model for the utility function,
\begin{align}
\label{eq:linutil}
U(\bx,\by) = \bw_*^{\top} \phi(\bx,\by),
\end{align}
where $\bw_* \in \mathbf{R}^N$ is an unknown parameter vector and $\phi: {\cal X} \times {\cal Y} \rightarrow \RR^N$ is a joint feature map such that $\|\phi(\bx,\by)\|_{\ell_2} \le R$ for any $\bx \in {\cal X}$ and $\by \in {\cal Y}$. Note that both $\bx$ and $\by$ can be structured objects.

We start by presenting and analyzing the most basic algorithm for the coactive learning model, which we call the {\em Preference Perceptron} (Algorithm \ref{perceptron}). The Preference Perceptron maintains a weight vector $\bw_t$ which is initialized to ${\mathbf 0}$. At each time step $t$, the algorithm observes the context $\bx_t$ and presents an object $\by$ that maximizes $\bw_t^\top \phi(\bx_t,\by)$. The algorithm then observes user feedback $\bby_t$ and the weight vector $\bw_t$ is updated in the direction $\phi(\bx_t,\bby_t) - \phi(\bx_t,\by_t)$.

\begin{algorithm}[tb]
\begin{algorithmic}
\STATE Initialize $\bw_1 \leftarrow {\mathbf 0}$
\FOR{$t=1$ {\bf to} $T$}
\STATE Observe $\bx_t$
\STATE Present $\by_t \leftarrow \argmax_{\by \in {\cal Y}} \bw_t^\top \phi(\bx_t,\by)$
\STATE Obtain feedback $\bby_t$
\STATE Update: $\bw_{t+1} \leftarrow \bw_{t} + \phi(\bx_t,\bby_t) - \phi(\bx_t,\by_t)$
\ENDFOR
\end{algorithmic}
\caption{\label{perceptron} Preference Perceptron.}
\vskip -0.001in
\end{algorithm}

\begin{theorem}
\label{thm:bound}
The average regret of the preference perceptron algorithm can be upper bounded, for any $\alpha \in (0,1]$ and for any $\bw_*$ as follows:
\begin{align}
\label{eq:percept-reg}
 REG_T  \le \frac{1}{\alpha T} \! \sum_{t=1}^T \xi_t + \frac{2R \|\bw_* \|}{ \alpha \sqrt{T}}.
\end{align}
\end{theorem}
\begin{proof}
First, consider $\|\bw_{T+1}\|^2$, we have,
\begin{align}
&\bw_{T+1}^\top\bw_{T+1} = \bw_{T}^\top \bw_{T} + 2 \bw_T^\top (\phi(\bx_T,\bby_T) - \phi(\bx_T,\by_T)) \nnnl
&+ ( \phi(\bx_T,\bby_T) - \phi(\bx_T,\by_T) )^\top(  \phi(\bx_T,\bby_T) - \phi(\bx_T,\by_T)  \nnnl
& \le \bw_{T}^\top\bw_{T} + 4 R^2  \le 4R^2T. \nn
\end{align}
On line one, we simply used our update rule from algorithm \ref{perceptron}. On line two, we used the fact that $\bw_T^\top  (\phi(\bx_T,\bby_T) - \phi(\bx_T,\by_T))  \le 0$ from the choice  of $\by_T$ in Algorithm \ref{perceptron} and that $\|\phi(\bx,\by)\| \le R$.  Further, from the update rule in algorithm \ref{perceptron}, we have,
\begin{align}
\bw_{T+1}^\top \bw_* &= \bw_{T}^\top\bw_* + (\phi(\bx_T,\bby_T) - \phi(\bx_T,\by_T))^\top\bw_* \nnnl
&= \sum_{t=1}^T  \left(  U(\bx_t,\bby_t) - U(\bx_t,\by_t) \right).
\end{align}
We now use the fact that $\inner{\bw_{T+1}}{\bw_*} \le \| \bw_* \| \| \bw_{T+1} \|$ (Cauchy-Schwarz inequality), which implies
$$  \sum_{t=1}^T    \left(  U(\bx_t,\bby_t) - U(\bx_t,\by_t) \right)  \le 2R\sqrt{T}\|\bw_*\|.$$
From the $\alpha$-informative modeling of the user feedback in \eqref{eq:inf-feedback-relax}, we have
$$  \alpha \sum_{t=1}^T \left(U(\bx_t,\by^*_t) - U(\bx_t,\by_t)\right)  - \sum_{t=1}^T \xi_t \le 2R\sqrt{T}\|\bw_*\|,$$
from which the claimed result follows.
\end{proof}



The first term in the regret bound denotes the quality of feedback in terms of violation of the strict $\alpha$-informative feedback. In particular, if the user feedback is strictly $\alpha$-informative, then all slack variables in \eqref{eq:percept-reg} vanish and $REG_T  = {\cal O}(1/\sqrt{T})$.

Though user feedback is modeled via $\alpha$-informative feedback, the algorithm itself does not require the knowledge of $\alpha$; $\alpha$ plays a role only in the analysis. 

Although the preference perceptron appears similar to the standard perceptron for multi-class classification problems, there are key differences. First, the standard perceptron algorithm requires the {\em true label $\by^*$} as feedback, whereas much weaker feedback $\bby$ suffices for our algorithm.  Second, the standard analysis of the perceptron bounds the number of mistakes made by the algorithm based on margin and the radius of the examples. In contrast, our analysis bounds a different regret that captures a graded notion of utility. 


%% file: extensions.tex
An appealing aspect of our learning model is that several interesting extensions are possible. We discuss some of them  in the rest of this  section. 

\subsection{Lower Bound}
We now show that the upper bound in Theorem \ref{thm:bound} cannot be improved in general.

\begin{lemma}
For any coactive learning algorithm ${\cal A}$ with linear utility, there exist  $\bx_t$, objects ${\cal Y}$  and $\bw_*$ such that  $REG_T$ of  ${\cal A}$  in $T$ steps is $\Omega(1/\sqrt{T})$.
\end{lemma}
\begin{proof}
Consider a problem where ${\cal Y} = \{-1,+1\}, {\cal X} = \{\bx \in \mathbf{R}^T: \|\bx\|=1\}$. Define the joint feature map as $\phi(\bx,\by) = \by \bx$. Consider $T$ contexts $\be_1, \ldots , \be_T$ such that $\be_j$ has only the $j^{th}$ component equal to one and all the others equal to zero. Let $\by_1, \ldots \by_T$ be the sequence of outputs of ${\cal A}$ on contexts $\be_1, \ldots , \be_T$. Construct $\bw_* = [-\by_1/\sqrt{T}~-\by_2/\sqrt{T} \cdots -\by_T/\sqrt{T}]^\top$, we have for this construction $\|\bw_*\| = 1$. Let the user feedback on the $t^{th}$ step be $-\by_t$. With these choices, the user feedback is always $\alpha$-informative with $\alpha=1$ since $\by_t^* = -\by_t$. Yet, the regret of the algorithm is $\frac{1}{T} \sum_{t=1}^T (\bw_*^\top\phi(\be_t,\by_t^*) - \bw_*^\top\phi(\be_t,\by_t)) = \Omega(\frac{1}{\sqrt{T}})$.
\end{proof}\vspace*{-0.5cm}

\subsection{Batch Update}
In some applications, due to high volumes of feedback, it might not be possible to do an update after every round. For such scenarios, it is natural to consider a variant of Algorithm~\ref{perceptron} that makes an update every $k$ iterations; the algorithm simply uses $\bw_{t}$ obtained from the previous update until the next update. It is easy to show the following regret bound for batch updates:
\begin{align}
REG_T  \le \frac{1}{\alpha T} \! \sum_{t=1}^T \xi_t + \frac{2R \|\bw_* \| \sqrt{k}}{ \alpha \sqrt{T}}. \nn
\end{align}
 
\subsection{Expected $\alpha$-Informative Feedback} \label{sec:expanalysis}
So far, we have characterized user behavior in terms of deterministic feedback actions. However, if a bound on the expected regret suffices, the weaker model of Expected $\alpha$-Informative Feedback from Equation~(\ref{eq:exp-feedback-relax}) is applicable. 

\begin{corollary}
\label{cor:bound}
Under  expected $\alpha$-informative feedback model, the expected regret (over user behavior distribution) of the preference perceptron algorithm can be upper bounded as follows:
\begin{align}
\label{eq:percept-expct-reg}
\expct[REG_T]  \le \frac{1}{\alpha T} \! \sum_{t=1}^T  \bar{\xi}_t + \frac{2R \|\bw_* \|}{ \alpha \sqrt{T}}.
\end{align}
\end{corollary}
The above corollary can be proved by following the argument of Theorem \ref{thm:bound}, but taking expectations over user feedback:
$\expct [\bw_{T+1}^\top\bw_{T+1}] 
= \expct[ \bw_{T}^\top \bw_{T}] + \expct[ 2 \bw_T^\top (\phi(\bx_T,\bby_T) - \phi(\bx_T,\by_T))   ]
+ \expct_T [( \phi(\bx_T,\bby_T) - \phi(\bx_T,\by_T) )^\top(  \phi(\bx_T,\bby_T) - \phi(\bx_T,\by_T)] 
\le \expct[\bw_T^\top \bw_T] + 4R^2.$  In the above, $\expct$ denotes expectation over all user feedback $\bby_t$ given $\by_t$ under the context $\bx_t$. It follows that  $\expct[\bw_{T+1}^\top \bw_{T+1}] \le 4TR^2$.


Applying Jensen's inequality on the concave function $\sqrt{\cdot}$, we get:
$\expct[\bw_T^\top\bw_*] \le \| \bw_* \| \expct[ \| \bw_T \|] \le \| \bw_* \| \sqrt{\expct[  \bw_T^\top\bw_T ]}.$ 
The corollary follows from the definition of expected $\alpha$-informative feedback.

\subsection{Convex Loss Minimization}
\label{sec:convex}
We now generalize our results to minimize convex losses defined on the linear utility differences.  We assume that at every time step $t$, there is an (unknown) convex loss function $c_t: \mathbf{R} \rightarrow \mathbf{R}$ which determines the loss $c_t(U(\bx_t,\by_t) - U(\bx_t,\by_t^*))$ at time $t$. The functions $c_t$ are assumed to be non-increasing. Further, sub-derivatives of the $c_t$'s are assumed to be bounded (i.e., $c_t'(\theta) \in [-G,0]$ for all $t$ and for all $\theta \in \mathbf{R}$).  The vector $\bw_*$ which determines the utility of $\by_t$ under context $\bx_t$ is assumed from a closed and bounded convex set ${\cal B}$ whose diameter is denoted as $|{\cal B}|$.

Algorithm \ref{convex}  minimizes the average convex loss. There are two differences between this algorithm and Algorithm \ref{perceptron}. Firstly, there is a rate $\eta_t$ associated with the update at time $t$. Moreover, after every update, the resulting vector $\bbw_{t+1}$ is projected back to the set ${\cal B}$. We have the following result for Algorithm \ref{convex}, a proof of which is provided in an extended version of this paper \cite{coactivearxiv}.

\begin{algorithm}[tb]
\begin{algorithmic}
\STATE Initialize $\bw_1 \leftarrow {\mathbf 0}$
\FOR{$t=1$ {\bf to} $T$}
\STATE Set $\eta_t \leftarrow \frac{1}{\sqrt{t}}$
\STATE Observe $\bx_t$
\STATE Present $\by_t \leftarrow \argmax_{\by \in {\cal Y}} \inner{\bw_t}{\phi(\bx_t,\by)}$
\STATE Obtain feedback $\bby_t$
\STATE Update: $\bbw_{t+1} \leftarrow \bw_{t} + \eta_t G  (\phi(\bx_t,\bby_t) - \phi(\bx_t,\by_t) )$
\STATE Project: $\bw_{t+1} \leftarrow \arg \min_{\bu \in {\cal B}} \|\bu - \bbw_{t+1}\|^2$
\ENDFOR
\end{algorithmic}
\caption{\label{convex} Convex Preference Perceptron.}
\end{algorithm}
\vskip -0.2in

\begin{theorem}
\label{thm:conv-bound}
For the convex preference perceptron, we have, for any $\alpha \in (0,1]$ and any $\bw_* \in {\cal B}$,
\begin{align}
\label{eq:conv-bound}
& \frac{1}{T} \sum_{t=1}^T c_t(U(\bx_t,\by_t) - U(\bx_t,\by_t^*) )
\le  \frac{1}{T} \sum_{t=1}^T c_t \left( 0 \right) \nnnl
 + &\frac{2G}{\alpha T} \sum_{t =1}^T \xi_t +  \frac{1}{\alpha}\left( \frac{|{\cal B}|G}{2\sqrt{T}}\right.  \left. + \frac{|{\cal B}|G}{T} + \frac{4R^2G}{\sqrt{T}}   \right).
\end{align}
\end{theorem}

In the bound \eqref{eq:conv-bound}, $c_t(0)$ is the minimum possible convex loss since $U(\bx_t,\by_t) - U(\bx_t,\by_t^*)$ can never be greater than zero by definition of $\by_t^*$. Thus the theorem upper bounds the average convex loss via the minimum achievable loss and the quality of feedback. Like the previous result (Theorem \ref{thm:bound}), under strict $\alpha$-informative feedback, the average loss  approaches the best achievable loss at ${\cal O}(1/\sqrt{T})$ albeit with larger constant factors.


%% file: expts.tex
\section{Experiments}
We empirically evaluated the Preference Perceptron algorithm on two datasets. The two experiments differed in the nature of prediction and feedback. While the algorithm operated on structured objects (rankings) in one experiment, atomic items (movies)  were presented and received as feedback in the other. 

\subsection{Structured Feedback: Learning to Rank}

We evaluated our Preference Perceptron algorithm on the Yahoo! learning to rank dataset \cite{ylr11}. This dataset consists of query-url feature vectors (denoted as $\bx^q_i$ for query $q$ and URL $i$), each with a relevance rating $r^q_i$ that ranges from zero (irrelevant) to four (perfectly relevant).  
To pose ranking as a structured prediction problem, we defined our joint feature map as follows:
\begin{align}
\label{eq:ranking-utility}
\bw^\top\phi(q,\by) =  \sum_{i=1}^5 \frac{\bw^\top \bx^q_{\by_i}}{\log(i+1)}.
\end{align}
In the above equation, $\by$ denotes a ranking such that $\by_i$ is the index of the URL which is placed at position $i$ in the ranking. Thus, the above measure
considers the top five URLs for a query $q$ and computes a score based on a graded relevance.  Note that the above utility function  defined via the feature-map is analogous to DCG@5 (see e.g. \cite{Manning/etal/08}) after replacing the relevance label with a linear prediction based on the features. 

For query $q_t$ at time step $t$, the Preference Perceptron algorithm presents the ranking $\by^q_t$ that maximizes $\bw_t^\top \phi(q_t,\by)$. Note that this merely amounts to sorting documents by the scores $\bw_t^\top \bx^{q_t}_i$, which can be done very efficiently. The utility regret in Eqn. \eqref{eq:linregret}, based on the definition of utility in \eqref{eq:ranking-utility}, is given by $\frac{1}{T} \sum_{t=1}^T\bw_*^{\top}(\phi(q_t,\by^{q_t*}) - \phi(q_t,\by^{q_t})) $. Here $\by^{q_t*}$ denotes the optimal ranking with respect to $\bw_*$, which is the best least squares fit to the relevance labels from the features using the entire dataset. Query ordering was randomly permuted twenty times and we report average and standard error of the results. 


\subsubsection{Strong Vs Weak Feedback}
The goal of the first experiment was to see how the regret of the algorithm changes with feedback quality. To get feedback at different quality levels $\alpha$, we used the following mechanism.  Given the predicted ranking $\by_t$, the user would go down the list until she found five URLs such that, when placed at the top of the list, the resulting $\bby_t$ satisfied the strictly $\alpha$-informative feedback condition w.r.t. the optimal $\bw_*$.  

\begin{figure}[h!!!]
\begin{center}
\hspace*{-0.3cm}\includegraphics[width=3.2in]{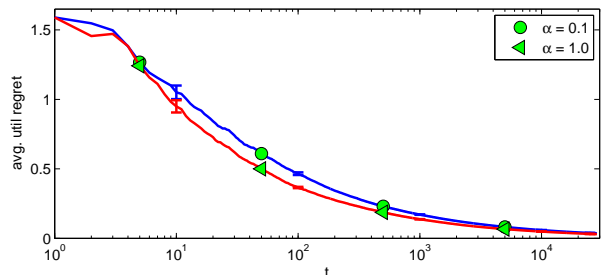}
\vskip -0.2in
\caption{\label{fig:clean-feedback} Regret based on strictly $\alpha$-informative feedback.}
\end{center}
\vskip -0.15in
\end{figure}

Figure \ref{fig:clean-feedback} shows the results for this experiment for two different $\alpha$ values. As expected, the regret with $\alpha=1.0$ is lower compared to the regret with respect $\alpha=0.1$. Note, however, that the difference between the two curves is much smaller than a factor of ten. This is because strictly $\alpha$-informative feedback is also strictly $\beta$-informative feedback for any $\beta \le \alpha$. So, there could be several instances where user feedback was much stronger than what was required. As expected from the theoretical bounds, since the user feedback is based on a linear model with no noise, utility regret approaches zero.

\subsubsection{Noisy Feedback}
In the previous experiment, user feedback was based on actual utility values computed from the optimal $\bw_*$. We next make use of the actual relevance labels provided in the dataset for user feedback. Now, given a ranking for a query, the user would go down the list inspecting the top 10 URLs (or all the URLs if the list is shorter) as before. Five URLs with the highest relevance labels ($r^q_i$) are placed at the top five locations in the user feedback. Note that this produces noisy feedback since no linear model can perfectly fit the relevance labels on this dataset.

As a baseline, we repeatedly trained a conventional Ranking SVM\footnote{{\tt http://svmlight.joachims.org}}. At each iteration, the previous SVM model was used to present a ranking to the user. The user returned a ranking based on the relevance labels as above. The pairs of examples $(q_t,\by^{q_t}_{svm})$ and $(q_t,\bby^{q_t}_{svm})$ were used as training pairs for the ranking SVMs. Note that training a ranking SVM after each iteration would be prohibitive, since it involves solving a quadratic program and cross-validating the regularization parameter $C$. Thus, we retrained the SVM whenever 10\% more examples were added to the training set. The first training was after the first iteration with just one pair of examples (starting with a random $\by^{q_1}$), and the $C$ value was fixed at $100$ until there were $50$ pairs of examples, when reliable cross-validation became possible. After there were more than $50$ pairs in the training set, the $C$ value was obtained via five-fold cross-validation. Once the $C$ value was determined, the SVM was trained on all the training examples available at that time. The same SVM model was then used to present rankings until the next retraining. 

\begin{figure}[h!!]
\begin{center}
\hspace*{-0.3cm}\includegraphics[width=3.2in]{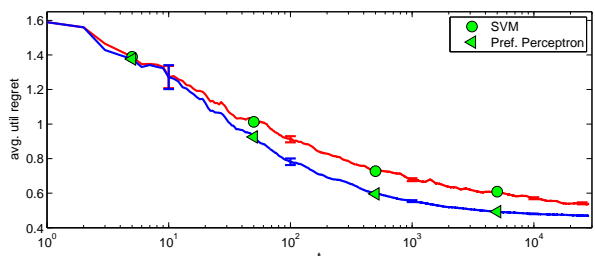}
\vskip -0.2in
\caption{\label{fig:rel-feedback} Regret vs time based on noisy feedback.}
\end{center}
\vskip -0.15in
\end{figure}

Results of this experiment are presented in Figure \ref{fig:rel-feedback}.   Since the feedback is now based on noisy relevance labels, the utility regret converges to a non-zero value as predicted by our theoretical results. Over most of the range, the Preference Perceptron performs significantly\footnote{The error bars are extremely tiny at higher iterations.} better than the SVM. Moreover, the perceptron experiment took around 30 minutes to run, whereas the SVM experiment took about 20 hours on the same machine.  We conjecture that the regret values for both the algorithms can be improved with better features or kernels,  but these extensions are orthogonal to the main focus of this paper. 

\subsection{Item Feedback: Movie Recommendation}

In contrast to the structured prediction problem in the previous section, we now evaluate the Preference Perceptron on a task with atomic predictions, namely movie recommendation. In each iteration a movie is presented to the user, and the feedback consists of a movie as well. We use the MovieLens dataset, which consists of a million ratings over 3090 movies rated by 6040 users. The movie ratings ranged from one to five. 

We randomly divided users into two equally sized sets. The first set was used to obtain a feature vector $\bm_j$ for each movie $j$ using the ``SVD embedding'' method for collaborative filtering (see \cite{Bell/Koren/07}, Eqn. (15)). The dimensionality of the feature vectors and the regularization parameters were chosen to optimize cross-validation accuracy on the first dataset in terms of squared error. For the second set of users, we then considered the problem of recommending movies based on the movie features $\bm_j$. This experiment setup simulates the task of recommending movies to a new user based on movie features from old users.


For each user $i$ in the second set, we found the best least squares approximation $\bw_{i*}^T \bm_j$ to the user's utility functions on the available ratings. This enables us to impute utility values for movies that were not explicitly rated by this user. Furthermore, it allows us to measure regret for each user as $\frac{1}{T} \sum_{t=1}^T \bw_{i*}^\top(\bm_{t*} - \bm_t)$, which is the average difference in utility between the recommended movie $\bm_t$ and the best available movie $\bm_{t*}$.  We denote the best available movie at time $t$ by $\bm_{t*}$, since in this experiment, once a user gave a particular movie as feedback, both the recommended movie and the feedback movie were removed from the set of candidates for subsequent recommendations. 

\subsubsection{Strong Vs Weak Feedback}

Analogous to the web-search experiments, we first explore how the performance of the Preference Perceptron changes with feedback quality $\alpha$. In particular, we recommended a movie with maximum utility according to the current $\bw_t$ of the algorithm, and the user returns as feedback a movie with the smallest utility that still satisfied strictly $\alpha$-informative feedback according to $\bw_{i*}$. For every user in the second set, the algorithm iteratively recommended 1500 movies in this way. Regret was calculated after each iteration and separately for each user, and all regrets were averaged over all the users in the second set.

\begin{figure}[h!!!]
\begin{center}
\hspace*{-0.3cm}\includegraphics[width=3.2in]{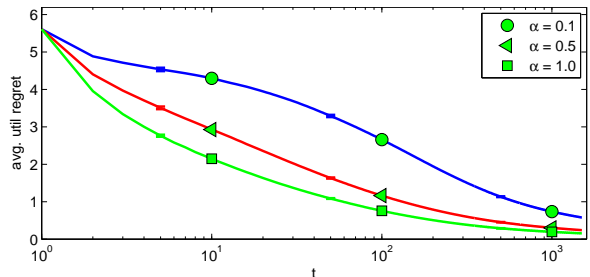}
\vskip -0.2in
\caption{\label{fig:movie} Regret for strictly $\alpha$-informative feedback.}
\end{center}
\vskip -0.15in
\end{figure}

Figure \ref{fig:movie} shows the results for this experiment. Since the feedback in this case is strictly $\alpha$-informative, the average regret in all the cases decreases towards zero as expected. Note that even for a moderate value of $\alpha$, regret is already substantially reduced after 10's of iterations. With higher $\alpha$ values, the regret converges to zero at a much faster rate than with lower $\alpha$ values.

\subsubsection{Noisy Feedback}

We now consider noisy feedback, where the user feedback does not necessarily match the linear utility model used by the algorithm. In particular, feedback is now given based on the actual ratings when available, or the score $\bu_{i*}^\top\bm_j$ rounded to the nearest allowed rating value. In every iteration, the user returned a movie with one rating higher than the one presented to her. If the algorithm already presented a movie with the highest rating, it was assumed that the user gave the same movie as feedback. 

\begin{figure}[h!!!]
\begin{center}
\hspace*{-0.3cm}\includegraphics[width=3.2in]{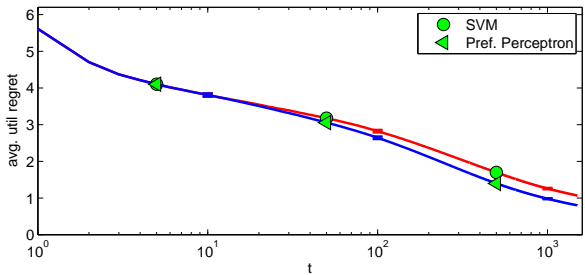}
\vskip -0.2in
\caption{\label{fig:movie2} Regret based on noisy feedback.}
\end{center}
\vskip -0.15in
\end{figure}

As a baseline, we again ran a ranking SVM. Like in the web-search experiment, it was retrained whenever 10\% more training data was added. The results for this experiment are shown in Figure~\ref{fig:movie2}. The regret of the Preference Perceptron is again significantly lower than that of the SVM, and at a small fraction of the computational cost.  